%% file: gnn.tex
\documentclass{article}
\PassOptionsToPackage{numbers, compress}{natbib}
\usepackage[preprint]{neurips_2019}
\usepackage[utf8]{inputenc} % allow utf-8 input
\usepackage[T1]{fontenc}    % use 8-bit T1 fonts
\usepackage{hyperref}       % hyperlinks
\usepackage{url}            % simple URL typesetting
\usepackage{booktabs}       % professional-quality tables
\usepackage{amsfonts}       % blackboard math symbols
\usepackage{nicefrac}       % compact symbols for 1/2, etc.
\usepackage{microtype}      % microtypography
\usepackage{xspace}
\usepackage{graphicx}
\usepackage[noend]{algpseudocode}
\usepackage{algorithm}
\usepackage{threeparttable}
\usepackage{amsmath}
\usepackage{amssymb}
\usepackage{amsthm}
\usepackage{cleveref}
\usepackage{subfig}
\usepackage{wrapfig}
\usepackage{enumitem}

\usepackage{xcolor}
\newcommand{\er}[1]{\mbox{\rm\em #1}}

\newcommand{\xg}{HAG\xspace}
\newcommand{\xgs}{HAGs\xspace}

\newcommand{\mw}[1] {\mathcal{\widehat{#1}}}
\newcommand{\m}[1] {\mathcal{#1}}
\newtheorem{theorem}{Theorem}
\DeclareMathOperator*{\argmax}{arg\,max}

\newcommand{\hide}[1]{}

\title{Redundancy-Free Computation Graphs for\\ Graph Neural Networks}

\author{%
  Zhihao Jia \\
  Stanford University \\
  \texttt{zhihao@cs.stanford.edu}\\
  \And
  Sina Lin \\
  Microsoft \\
  \texttt{silin@microsoft.com}
  \And
  Rex Ying \\
  Stanford University \\
  \texttt{rexying@stanford.edu}
  \And 
  Jiaxuan You\\
  Stanford University\\
  \texttt{jiaxuan@stanford.edu}
  \And
  Jure Leskovec\\
  Stanford University\\
  \texttt{jure@cs.stanford.edu}\\
  \And
  Alex Aiken\\
  Stanford University\\
  \texttt{aiken@cs.stanford.edu}\\
}

\begin{document}

\maketitle

\begin{abstract}
%Graph Neural Networks (GNNs) have revolutionized machine learning on graphs.
Graph Neural Networks (GNNs) are based on repeated aggregations of information across nodes' neighbors in a graph. However, because common neighbors are shared between different nodes, this leads to repeated and inefficient computations.
%means that many computations are repeated and inefficient.
%
We propose {\em Hierarchically Aggregated computation Graphs} (\xgs), a new GNN graph representation that explicitly avoids redundancy by managing intermediate aggregation results hierarchically, eliminating repeated computations and unnecessary data transfers in GNN training and inference.
We introduce an accurate cost function to quantitatively evaluate the runtime performance of different \xgs and use a novel \xg search algorithm to find optimized \xgs.
%We introduce a novel cost model and the develop an algorithm that provably finds the minimum-cost GNN computation graph.
%
Experiments show that the \xg representation significantly outperforms the standard GNN graph representation by increasing the end-to-end training throughput by up to 2.8$\times$ and reducing the aggregations and data transfers in GNN training by up to 6.3$\times$ and 5.6$\times$, while maintaining the original model accuracy.
\hide{ %% Jure
Existing graph neural networks (GNNs) use ordinary graph representation that directly connects each vertex in a graph with its neighbors.
Each vertex computes its activations by aggregating its neighbors independently, resulting in redundant computation and unnecessary data transfers.
In this paper, we propose \xg, a new graph representation that explicitly manages intermediate aggregation results hierarchically, which reduces redundant computation and eliminates unnecessary data transfers.
We introduce a cost model to quantitatively evaluate the runtime performance of different \xgs and use novel graph search algorithm to find highly optimized \xgs under the cost model.
Our evaluation shows that the \xg representation significantly outperforms the standard graph representations by reducing computation costs and memory accesses for neighborhood aggregations by up to 84\% and YY\%, and increasing end-to-end training throughput by up to 1.9$\times$, while maintaining the original model accuracy.
}
\vspace{-3mm}
\end{abstract}

\input{intro}
\input{related}
%\input{model}
\input{agraph}
\input{method}
\input{exp}
\section{Conclusion}
We have introduced \xg, a new GNN graph representation to eliminate redundant computation and data transfers in GNNs.
We propose a cost function to quantitatively evaluate the runtime performance of different \xgs and use a \xg search algorithm to find optimized \xgs.
Our experiments show that \xgs significantly outperform existing GNN-graphs by improving the end-to-end training performance and reducing the aggregations and data transfers in GNN training.
%We are planning to release our implementations upon acceptance to facilitate future research.
\bibliographystyle{unsrtnat}
\bibliography{bibliography}
\input{supply}

\end{document}

%% file: intro.tex
\section{Introduction}
\label{sec:intro}
Graph neural networks (GNNs) have shown state-of-the-art performance across a number of tasks with graph-structured data, such as social networks, molecule networks, and webpage graphs~\cite{GCN, GraphSAGE, DiffPool, GIN, CNLMF}. 
GNNs use a recursive neighborhood aggregation scheme --- in a GNN layer, each node aggregates its neighbors' activations from the previous GNN layer and uses the aggregated value to update its own activations.
The activations of the final GNN layer are used for prediction tasks, such as node classification, graph classification, or link prediction.

Due to the clustering nature of real-world graphs, different nodes in a graph may share a number of common neighbors.
For example, in webpage graphs, different websites under the same domain generally have a number of common links (i.e., neighbors).
As another example, in recommender systems, users in the same group may have interests in common items.

%\rex{However, Existing GNNs contain significant redundancies in their GNN computation graphs (instead of the current first sentence, maybe a however statement to summarize the paragraph?)}
%However, existing GNNs contain significant redundancies in their GNN computation graphs, which leads to repeated and unnecessary computation.
However, existing GNN representations do not capture these common neighbors in real-world graphs, leading to redundant and unnecessary computation in both GNN training and inference. 
%However, these common neighbors result in significant redundancies in GNN computation graphs, leading to repeated and unnecessary computation in both GNN training and inference.
In particular, existing GNN representations define computation in each GNN layer with a GNN {\em computation graph} (referred to as a GNN-graph). 
For each node $v$ in the input graph, the GNN-graph includes an individual tree structure to describe how to compute $v$'s activations by aggregating the previous-layer activations of $v$'s neighbors.
%Existing GNNs define computation in neighborhood aggregations with a GNN computation graph that includes an individual aggregation operator to aggregate each node's neighbors.
%%To perform neighborhood aggregations in a GNN layer, existing GNNs use a {\em standard} format for computation graphs with a number of individual aggregation operators, each of which aggregates the neighbors of a vertex in an input graph.
%Figure~\ref{fig:intro}a shows an example input graph. 
Figure~\ref{fig:intro}b shows the GNN-graph of the input graph in Figure~\ref{fig:intro}a;
for example, for node $A$, its neighbor's activations $h^{(k-1)}_B$, $h^{(k-1)}_C$ and $h^{(k-1)}_D$ from the layer $k-1$ are aggregated to compute new activations $h^{(k)}_A$ for the layer $k$ (see the top portion of Figure~\ref{fig:intro}b).
The new activations of the other nodes are computed similarly using the previous activations of their neighbors.
Notice that this representation results in redundant computation and data transfers.
In this small example, both $\{A,B\}$ and $\{C,D\}$ are aggregated twice.
In wider and mlulti-layer GNNs, the redundancies in existing GNN representations account for a significant fraction of all computation.

%Figure~\ref{fig:intro}b shows an example standard computation graph.
%This approach performs neighborhood aggregations independently on each vertex, resulting in redundant computation and data transfers since aggregations on the shared subsets of vertices are performed multiple times (e.g., both $\{A, B\}$ and $\{C, D\}$ are aggregated twice in Figure~\ref{fig:intro}b).

\begin{figure*}
    %\vspace{-3mm}
    \centering
    \includegraphics[scale=0.45]{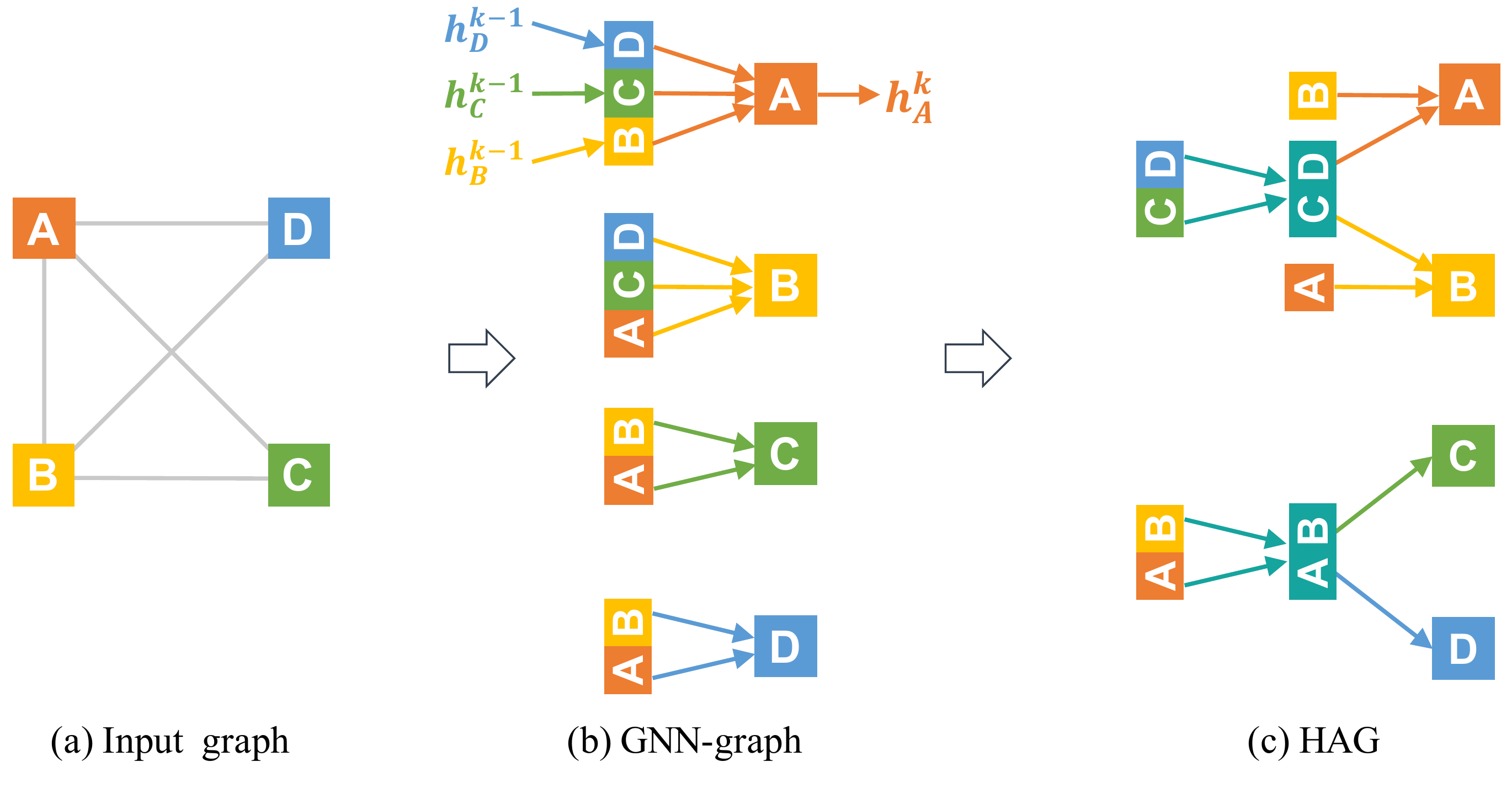}
    %\vspace{-2mm}
    \caption{Comparison between a GNN-graph and an equivalent \xg.
    (a) Input graph; (b) 1-layer GNN computation graph (GNN-graph); (c) \xg that avoids redundant computation.
    The GNN-graph computes new activations $h^{(k)}_v$ by aggregating the previous-layer activations of $v$'s neighbors.
    Because nodes in the input graph share common neighbors, the GNN-graph performs redundant computation (e.g., both $\{A, B\}$ and $\{C, D\}$ are aggregated twice). (c) By identifying common computational patterns, the \xg avoids repeated computation.
    }
    \label{fig:intro}
    %\vspace{-6mm}
\end{figure*}

In this paper, we propose a new GNN representation called {\em Hierarchically Aggregated computation Graphs} (\xgs).
%, which eliminate redundant computation and unnecessary data transfers in GNN computation graphs by hierarchically managing and reusing intermediate aggregation results. 
Figure~\ref{fig:intro}c shows one possible \xg for the input graph in Figure~\ref{fig:intro}a. %, which computes neighborhood aggregations hierarchically to reuse intermediate aggregations.
%Compared to standard computation graphs, \xgs 
\xgs are functionally equivalent to standard GNN-graphs (produce the same output), but represent common neighbors across different nodes using aggregation hierarchies, which eliminates redundant computation and unnecessary data transfers in both GNN training and inference.
%hierarchically combine repeated computation and eliminate unnecessary data transfers.
%, while preserving the exact same output as the standard computation graphs.
In addition, a \xg is agnostic to any particular GNN model, and can be used to eliminate redundancy for arbitrary GNNs.
%a new graph representation for GNNs called {\em hierarchical aggregation graphs} (\xgs), which eliminate redundant computation and reduce unnecessary data accesses in GNNs by reusing intermediate aggregation results hierarchically.
%\xg includes a number of intermediate {\em aggregation vertices}, each of which represents the aggregation result of a specific subset of vertices and can be reused for the neighborhood aggregations of multiple vertices.

%Compared to the original graph representation that directly link each vertex to all neighbors, our HSG representation eliminates redundant aggregations on XXX and reduces memory accesses by reusing intermediate subset vertices.

For a GNN-graph, there exist numerous equivalent \xgs with different aggregation hierarchies and runtime performance.
%Find \xgs with optimized performance is challenging due to the large space of possible \xgs.
Finding \xgs with optimized performance is challenging since the number of possible \xgs is exponential in the input graph size.
%Our goal is to find a \xg with optimized runtime performance while preserving original model accuracy. 
%To formalize the problem,
We introduce an accurate cost function to quantitatively estimate the performance of different \xgs and develop a novel \xg search algorithm to automatically find optimized \xgs.

Theoretically, we prove that the search algorithm can find \xgs with strong performance guarantees: (1) for GNNs whose neighborhood aggregations require a specific ordering on a node's neighbors, the algorithm can find a globally optimal \xg under the cost function; and (2) for other GNNs, the algorithm can find \xgs whose runtime performance is at least a $(1-1/e)$ approximation ($\approx 63\%$) of globally optimal \xgs using the submodularity property~\cite{mossel2007submodularity}.
Empirically, the algorithm finds highly optimized \xgs for real-world graphs, reducing the number of aggregations by up to 6.3$\times$.
%For GNN models whose aggregations require a specific ordering on a vertex's neighbors, the greedy algorithm is able to find a globally optimal \xg under the cost model.
%For other GNN models, our algorithm has no optimality guarantee but empirically reduces computation costs of neighborhood aggregations by up to 6.7$\times$.

%Existing deep learning frameworks such as TensorFlow and PyTorch train GNNs by translating GNN-graphs to sparse matrices and performing matrix operations.
%Besides being less efficient than \xgs, this approach does not consider graph structures in GNNs and disables a number of critical system optimizations for graphs (see Section~\ref{sec:impl}).

%Based on the above insights, we implemented \xg in new GNN framework we call \sys.
%The key difference between \sys and existing frameworks is that \sys explicitly manages graph structures in GNNs and reduces GNN training to a number of graph processing operations. 
%This allows \sys to directly benefit from system optimizations for graphs.

Our \xg abstraction maintains the predictive performance of GNNs but leads to much faster training and inference. We evaluate the performance of \xgs on five real-world datasets and along three dimensions: (a) end-to-end training and inference performance; (b) number of aggregations; and (c) size of data transfers.
%We evaluate the runtime performance of \sys with \xgs on five real-world graph datasets including social networks~\cite{GraphSAGE}, molecule networks~\cite{BZR, PPI}, and scientific collaboration datasets~\cite{COLLAB}.
Experiments show that \xgs increase the end-to-end training and inference performance by up to 2.8$\times$ and 2.9$\times$, respectively. 
%Experiments show that \sys significantly outperforms state-of-the-art deep learning frameworks for GNN training, with end-to-end speedups ranging from 4.6$\times$ to 15.3$\times$.
%First, \sys enables a number of system optimizations for graphs by reducing GNN training to a sequence of graph processing operations, which increases the training throughput by 3.7-5.5$\times$.
%Second, \sys uses the \xg representation to eliminate redundant aggregations and data transfers in GNN training, which further increases the training throughput by up to 2.8$\times$.
In addition, compared to GNN-graphs, \xgs reduce the number of aggregations and the size of data transfers by up to 6.3$\times$ and $5.6\times$, respectively.

%Compared to standard computation graphs, \xgs improves the training throughput of GNN models by up to 2.3$\times$. In addition, \xgs can reduce the computation costs and data transfers for neighborhood aggregations by up to 6.7$\times$ and 2.2$\times$, respectively.
%\ZJ{\xgs allow more efficient graph implementation. }

%\hide{
To summarize, our contributions are:
\begin{itemize}[leftmargin=*]
    \setlength\itemsep{0em}
    \item We propose \xg, a new GNN graph representation to eliminate redundant computation and data transfers in GNNs.
    \item We define a cost model to quantitatively evaluate the runtime performance of different \xgs and develop a \xg search algorithm to automatically find optimized \xgs.
    Theoretically, we prove that the \xg search algorithm at least finds a $(1-1/e)$-approximation of globally optimal \xgs under the cost model.
    \item We show that \xgs significantly outperform GNN-graphs by increasing GNN training and inference performance by up to XX$\times$ and YY$\times$, respectively, and reducing the aggregations and data transfers in GNN-graphs by up to 6.3$\times$ and 5.6$\times$, respectively.
\end{itemize}
%}

%% file: related.tex
\section{Related Work}
\label{sec:related}
{\bf Graph neural networks}
have been used to solve various real-world tasks with relational structures~\cite{GCN, GraphSAGE, DiffPool, GIN, CNLMF}.
FastGCN~\cite{FastGCN} and SGC~\cite{SGC} accelerate GNN training using importance sampling and removing nonlinearilities.
This paper solves an orthogonal problem: how to optimize GNN efficiency while maintaining network accuracy.
\xg is agnostic to any particular GNN model and provides a general approach that can be automatically applied to eliminate redundancy for arbitrary GNN models.

{\bf Join-trees} are a tree decomposition technique that maps a graph into a corresponding tree structure to solve optimization problems on the graph, such as query optimization~\cite{query_optimization}.
Although a join-tree provides a possible way to find optimal \xgs for a GNN-graph, its time complexity is exponential in the {\em treewidth} of a GNN-graph~\cite{treewidth}, and real graphs tend to have very large treewidths.
For example, \cite{adcock2016tree} shows that the treewidth of real-world social networks grow linearly with the network size, making it infeasible to use join-trees to find optimal \xgs.
%\vspace{-1mm}

\paragraph{Computation reduction in neural networks.}
Several techniques have been proposed to reduce computation in neural networks, including weights pruning~\cite{Han1} and quantization~\cite{Han2}. 
%For example, ~\citet{Han1} presents a weight pruning algorithm to iteratively remove weak connections in a network. 
%As another example, ~\citet{Han2} proposes a deep compression technique to reduce network computation by training on low precision weights.
These techniques reduce computation at the cost of modifying networks, resulting in decreased accuracy (as reported in these papers).
By contrast, we propose a new GNN representation that accelerates GNN training by eliminating redundancy in GNN-graphs while maintaining the original network accuracy.
%
%\vspace{-2mm}
%

%% file: agraph.tex
\section{Hierarchically Aggregated Computation Graphs (\xgs)}
\label{subsec:graph}

\begin{table*}[t]
%\vspace{-1mm}
\caption{Existing GNNs described in our abstraction. GraphSAGE-P and GraphSAGE-LSTM are the pooling and LSTM variants of GraphSAGE, respectively. $\sigma$ and $\er{max}$ indicate element-wise non-linear activation and max functions.
For sequential \textproc{Aggregate}, $v_i$ denotes the $i$-th in-neighbor of node $v$.
}
%\vspace{-2mm}
\label{tab:gnns}
\begin{threeparttable}
\resizebox{\textwidth}{!}{
\begin{tabular}{lll}
\hline
{\bf GNN} & {\bf $\textproc{Aggregate}(\{h^{(k-1)}_u | u \in \mathcal{N}(v)\})$} & {\bf $\textproc{Update}(a^{(k)}_v, h^{(k-1)}_v)$}\\
\hline
\multicolumn{3}{c}{Set \textproc{Aggregate}} \\
\hline
GCN~\cite{GCN} & $a^{(k)}_v = \sum_{u\in\mathcal{N}(v)}{h^{(k-1)}_u} $ & $h^{(k)}_v = \sigma(W^{(k)} \cdot \frac{a^{(k)}_v + h^{(k-1)}_v}{|\mathcal{N}(v)| + 1})$ \\
%GIN~\cite{GIN} & $a^{(k)}_v = \sum_{u\in\mathcal{N}(v)}{h^{(k-1)}_u} $ & $h^{(k)}_v = \sigma\big(W \cdot \Big((1 + \epsilon^{(k)})h^{(k-1)}_v + a^{(k)}_v\big) \Big)$ \\
GraphSAGE-P~\cite{GraphSAGE} & $a^{(k)}_v = {\er{max}}_{u\in\mathcal{N}(v)}\{\sigma(W^{(k)}_1 \cdot h^{(k-1)}_u)\}$ & $h^{(k)}_v = \sigma\big(W^{(k)}_2 \cdot (a^{(k)}, h^{(k-1)}_v)\big)$ \\
%GraphSAGE~\cite{GraphSAGE} & $a^{(k)}_v = \frac{1}{|\mathcal{N}(v)|}\sum_{u\in\mathcal{N}(v)}{h^{(k-1)}_u}$ & $h^{(k)}_v = \sigma\big(W^{(k)} \cdot (a^{(k)}_v, h^{(k-1)}_v)\big)$\\
%\hline
%\multicolumn{3}{c}{Idempotent \textproc{Aggregate}} \\
%\hline
%GCN-P~\cite{GCN} & $a^{(k)}_v = {\er{max}}_{u\in\mathcal{N}(v)}\{h^{(k-1)}_u\} $ & $h^{(k)}_v = \sigma\big(W^{(k)} \cdot (a^{(k)}_v | h^{(k-1)}_v)\big)$ \\
\hline
\multicolumn{3}{c}{Sequential \textproc{Aggregate}} \\
\hline
GraphSAGE-LSTM~\cite{GraphSAGE} & $a^{(k)}_v = \er{LSTM}(h^{(k-1)}_{v_1},...,h^{(k-1)}_{v_\mathcal{N}})$ & $h^{(k)}_v = \sigma\big(W^{(k)} \cdot (a^{(k)}_v, h^{(k-1)}_v)\big)$\\
$N$-ary Tree-LSTM~\cite{TreeLSTM} & $a^{(k)}_v = \er{Tree-LSTM-Agg}(h^{(k-1)}_{v_1},...,h^{(k-1)}_{v_\mathcal{N}})$ & $h^{(k)}_v = \er{Tree-LSTM-Update}(a^{(k)}_v, h^{(k-1)}_v)$\\
\hline
\end{tabular}
}
\end{threeparttable}
%\vspace{-3mm}
\end{table*}

\begin{algorithm}[t]
\caption{An abstraction for GNNs. $\mathcal{V}$ is the set of nodes in an input graph, and $\mathcal{N}(v)$ denotes the set of neighbors for node $v$.}
\label{alg1}
%\footnotesize
\begin{algorithmic}[1]
\State $h^{(0)}_v = x_v, \forall v \in \mathcal{V}$
\For {$k= 1 \textrm{ to } K$}
\For {$v \in {\mathcal{V}}$}
\State $a^{(k)}_v \leftarrow \Call{Aggregate}{\{h^{(k-1)}_u | u \in \mathcal{N}(v)\}}$
\State $h^{(k)}_v \leftarrow \Call{Update}{a^{(k)}_v, h^{(k-1)}_v}$
\EndFor
\EndFor
\State
\State {\bf Goal:} minimize $\mathcal{L}(\{h^{(K)}_v | v \in \mathcal{V}\})$
\end{algorithmic}
\end{algorithm}

\paragraph{GNN abstraction.}
A GNN takes an input graph and node features as inputs and iteratively learns representations for individual nodes over the entire graph through a number of GNN layers.
%Each GNN layer consists of two steps.
Algorithm~\ref{alg1} shows an abstraction for GNNs: $h^{(k)}_v$ is the learned activations of node $v$ at layer $k$, and we initialize $h^{(0)}_v$ with input node features $x_v$.
At the $k$-th layer, $a^{(k)}_v$ denotes the aggregated activations of $v$'s neighbors, which is combined with $h^{(k-1)}_v$ to compute an updated activation $h^{(k)}_v$.
The learned node activations of the final layer (i.e., $h^{(K)}_v$) are used for predictions, and a GNN model generally minimizes a loss function $\mathcal{L}$ that takes the final node activations as inputs (line 6).

Existing GNN models use a GNN {\em computation graph} (GNN-graph) to describe the computation in each GNN layer, as shown in Figure~\ref{fig:intro}b.
For each node $v$ in the input graph, the GNN-graph includes an individual tree structure to define how to compute the activations $h_v^{(k)}$ of node $v$ by aggregating the previous-layer activations of $v$'s neighbors (i.e., $h^{(k-1)}_u, u \in \mathcal{N}(v)$).
GNN-graphs are efficient at expressing direct neighborhood relations between nodes, 
but are not capable of capturing common neighbors across multiple nodes, leading to redundant computation in GNN training and inference.

\subsection{\xg Definition}
%GNNs learn graph topology by aggregating each node's neighbors in each layer (i.e., \textproc{Aggregate} in Algorithm~\ref{alg1}).
We propose {\em Hierarchically Aggregated computation Graphs} (\xgs) for GNNs, which eliminate redundancy in GNN-graphs by hierarchically managing and reusing intermediate aggregation results.
%Compared to the original graph representation, \xg reduces both computation and data transfer costs in GNNs.
%Compared to an standard computation graph, a
Compared to a GNN-graph, a \xg includes a new set of {\em aggregation} nodes, each of which represents the intermediate aggregations result for a subset of nodes (i.e., aggregation on a subset of $h^{(k-1)}_v$).
Similar to edges in GNN-graphs, an edge $(u, v)$ in a \xg denotes an aggregation relation --- computing $v$'s activations requires aggregating $u$'s activations.

%A \xg $\mw{G} = (\langle\m{V}_S, \m{V}_A, \m{V}_I\rangle, \mw{E})$ includes a third set of nodes $\m{V}_I$, each of which represents the intermediate aggregation results for a subset of nodes (i.e., aggregation on a subset of $h^{(k-1)}_v$).
%A \xg $\mw{G} = (\langle\m{V}_S, \m{V}_A, \m{V}_I\rangle, \mw{E})$ is a directed acyclic graph with three types of nodes 
%(i.e., $\mathcal{\widehat{V} = \widehat{V}_S + \widehat{V}_D + \widehat{V}_I}$). 
%First, for each node $v$ in a given training graph, $\mathcal{\widehat{V}}_S$ includes a {\em source node} corresponding to $v$'s activations at the previous layer (i.e, $h^{(k-1)}_v$). 
%Second, for each node $v$ in the training graph, $\mathcal{\widehat{V}}_E$ has a {\em destination node} corresponding to the aggregated activations of $v$'s neighbors (i.e., $a^{(k)}_v$).
%Finally, $\mathcal{\widehat{V}}_I$ contains a number of intermediate {\em subset node}, each of which is the aggregated activations for a subset of nodes (i.e., aggregation on a subset of $h^{(k-1)}_v$).

%Edges in $\mathcal{H}$ denotes aggregation --- all in-edges of each node are aggregated and used as the node's representation.
%First, $\er{depth}(v)$ is the length of the longest path from a source node to $v$, which describes the depth of a node in the hierarchy.
%$$
%\er{depth}(v) = \begin{cases}
%0 & v \in \widehat{\mathcal{V}}_S \\
%\max_{(u, v) \in \widehat{\mathcal{E}}} \{\er{depth}(u) + 1\} & v \in %\mathcal{\widehat{V}}_I \cup \mathcal{\widehat{V}}_D
%\end{cases}
%$$

Our \xg abstraction is general and applicable to many existing GNN models.
Table~\ref{tab:gnns} shows how to use our abstraction to define existing GNNs, which can be further divided into two categories.
%based on their \textproc{Aggregate} functions.

\begin{itemize}[leftmargin=*]
\setlength\itemsep{0em}
\item {\bf Set \textproc{Aggregate}}. Most GNNs assume the neighbors of a node have {\em no ordering}, and the aggregations are {\em associative} and {\em commutative} operations that are invariant to the order in which the aggregations are performed. Examples include GCN with summation aggregations and GraphSAGE-P with element-wise pooling aggregations (Table~\ref{tab:gnns}).
Note that set aggregations in GNNs are designed to be order invariant and thus can be performed in a hierarchical fashion as we do in \xgs.
\hide{
%Note that aggregation functions in GNNs are designed to be order invariant which is exactly the property we need for \xgs  to work. 
%Essentially, set aggregations are associative and commutative and aggregations can thus be performed in a hierarchical fashion as we do in \xgs.
}
%\ZJ{We define an \textproc{Aggregate} function to be {\em unordered} if the aggregation results are permutation invariant.}
%and requires each element to be aggregated exactly once.
%\item {\bf Idempotent \textproc{Aggregate}}. A second class of GNNs assume {\em no ordering} on the neighbors of a node and allows each neighbor to be aggregated {\em multiple} times.
%For example, both GCN-P and GraphSAGE-P uses an element-wise max-pooling aggregator, which allows aggregating some neighbors multiple times and preserves the same outputs. 
%We call this an {\em idempotent} \textproc{Aggregate}.

\item {\bf Sequential \textproc{Aggregate}}. Another class of GNNs require a specific ordering of a node's neighbors and the aggregations are not commutative. 
Examples include $N$-ary Tree-LSTM~\cite{TreeLSTM} and the LSTM variant of GraphSAGE~\cite{GraphSAGE}.
However, \xgs can be applied in the case of sequential aggregations as well. 
Rather than identifying common subsets of neighbors, we identify the common prefixes of the sequence of aggregated nodes, which can then be reused among nodes.
%For GNNs with a {\em sequential} \textproc{Aggregate}, the neighbor set $\mathcal{N}(v)$ is an ordered list.
%describing the order in which \textproc{Aggregate} should be performed on the neighbors.
\end{itemize}

%\ZJ{say why we need to define two types of aggregators}

%{\bf Formal definition of \xgs.}
We shall use $\m{V}$ to denote the nodes in the input graph and use $\m{V}_A$ to denote the aggregation nodes added in a \xg.
The standard GNN-graph representation can be considered as a special case in the \xg representation with no intermediate aggregation nodes (i.e., $\m{V}_A = \emptyset$). 
We further define two additional functions for each node:

First, $\er{aggr}(v)$ is the aggregation results of node $v$:
$$
\er{aggr}(v) = \begin{cases}
h_{v}^{(k-1)} & \mw{N}_v = \emptyset \\
\textproc{Aggregate}(\{\er{aggr}(u) | u \in \mw{N}_v\}) & \mw{N}_v \neq \emptyset
\end{cases}
$$
where $\mw{N}_v$ denotes the in-neighbors of node $v$ in a \xg.
Note that $\er{aggr}(\cdot)$ is recursively defined, and there exists a sequential ordering to evaluate $\er{aggr}(v)$ for all nodes since each \xg is acyclic.

\hide{
First, $\er{aggr}(v)$ is the result of aggregating $v$'s in-neighbors with an \textproc{Aggregate} function.
For a source node $v$ with no in-neighbors, $\er{aggr}(v)$ is $v$'s activations in the previous layer:
$$
\er{aggr}(v) = \begin{cases}
h^{(k-1)}_v & v \in \m{V}_S \\
\textproc{Aggregate}(\{\er{aggr}(u) | (u, v) \in \mathcal{\widehat{E}}\}) & v \in \m{V}_A \cup \m{V}_I
\end{cases}
$$
Note that $\er{aggr}(\cdot)$ is recursively defined, and there exist a sequential ordering to evaluate $\er{aggr}(v)$ for all nodes since $\widehat{\mathcal{G}}$ is acyclic.
}

Second, we use $\er{cover}(v)$ to describe how to compute $\er{aggr}(v)$ by using the input activations $h^{(k-1)}_u$ from the previous layer.
%we define $\er{cover}(v)$ as the set of nodes whose activations are aggregated to compute $\er{aggr}(v)$.
\begin{equation}
\er{aggr}(v) = \textproc{Aggregate}(\{h^{(k-1)}_u | u \in \er{cover}(v)\}
\end{equation}
$\er{cover}(v)$ defines the coverage of node $v$ in a \xg. For the \xg example in Figure~\ref{sec:intro}c, $\er{cover}(A) =  \{B, C, D\}$ because $h^{(k-1)}_B$, $h^{(k-1)}_C$, and $h^{(k-1)}_D$ are used as inputs to compute $h^{(k)}_A$.

For a set \textproc{Aggregate}, $\er{cover}(\cdot)$ is an unordered set: % and can be calculated with the following equation.
\begin{equation}
\label{eqn1}
\er{cover}(v) = \begin{cases}
\{v\} & \mw{N}_v = \emptyset \\
\{w | \exists u \in \mw{N}_v: w \in \er{cover}(u)\} & \mw{N}_v \neq \emptyset
\end{cases}
\end{equation}

For a sequential \textproc{Aggregate}, $\er{cover}(\cdot)$ is an ordered list:
\begin{equation}
\label{eqn2}
\er{cover}(v) = \big(\er{cover}(u_1), ..., \er{cover}(u_m)\big)
\end{equation}
where $u_1, ..., u_m$ are the ordered in-neighbors of $v$.
%Theorem~\ref{thm1} shows how to compute $\er{cover}(v)$ for different types of aggregators. We prove the theorem in Appendix. 

%\begin{theorem}
%\label{thm1}
%For a source node $v$, $\er{cover}(v) = \{v\}$ by definition.
%\begin{eqnarray*}
%\er{cover}_{\rm ord}(v) & = &\{ w | \exists! (u, v)\in\mathcal{\widehat{E}}: w \in \er{cover}_{\rm ord}(u) \} \\
%\er{cover}_{\rm ide}(v) & = &\{ w | \exists (u, v)\in\mathcal{\widehat{E}}: w \in \er{cover}_{\rm ide}(u) \} \\
%\er{cover}_{\rm seq}(v) & = &( \er{cover}_{\rm seq}(u_1), ... , \er{cover}_{\rm seq}(u_m))\\
%\end{eqnarray*}
%where $(u_1, v), ..., (u_m, v)$ are ordered in-edges of $v$.
%\end{theorem}

%\begin{eqnarray*}
%& \er{cover}({\textproc{Aggregate}}, v) \\
%& = \begin{cases}
%\{ w | \exists! (u, v)\in\mathcal{\widehat{E}}: w \in \er{cover}(\textproc{Aggregate}, u) \} & \textrm{\textproc{Aggregate} is standard} \\
%\{ w | \exists (u, v)\in\mathcal{\widehat{E}}: w \in \er{cover}(\textproc{Aggregate}, u) \} & \textrm{\textproc{Aggregate} is idempotent} \\
%( \er{cover}(\textproc{Aggregate}, u_1), ... , \er{cover}(\textproc{Aggregate}, u_m)) & \textrm{\textproc{Aggregate} is sequential}
%\end{cases}
%\end{eqnarray*}

\subsection{GNNs with \xgs}
\begin{algorithm}[t]
\caption{A GNN abstraction with \xgs. $\widehat{a}_v$ denotes the result of $\er{aggr}(v)$ at a GNN layer. We exclude layer index superscripts in $\widehat{a}_v$ to denote that $\widehat{a}_v$ does not need to be memorized for back propagation,
and its memory can be reused across all layers.}
\label{alg2}
%\footnotesize
\begin{algorithmic}[1]
\State $h^{(0)}_v = x_v, \forall v \in \m{V}$
\For {$k= 1 \textrm{ to } K$}
\For {$v \in {\m{V}}$}
\State $\widehat{a}_v \leftarrow h^{(k-1)}_v$
\EndFor
\For {$v \in \m{V}_A$}
\State $\widehat{a}_v \leftarrow \Call{Aggregate}{\{\widehat{a}_u | u \in \mw{N}_v\}}$
\EndFor
\For {$v \in {\m{V}}$}
\State $a^{(k)}_v \leftarrow \Call{Aggregate}{\{\widehat{a}_u | u \in \mw{N}_v\}}$
\State $h^{(k)}_v \leftarrow \Call{Update}{a^{(k)}_v, h^{(k-1)}_v}$
\EndFor
\EndFor
\end{algorithmic}
\end{algorithm}

Existing GNNs are defined with GNN-graphs as shown in Algorithm~\ref{alg1}.
We extend the GNN abstraction in Algorithm~\ref{alg2} to make it also applicable to \xgs.
%Algorithm~\ref{alg2} shows the extended GNN abstraction for \xgs.
The extension does not require any modification to a GNN model, and the only difference is how to compute neighborhood aggregations (i.e., $a^{(k)}_v$) in each GNN layer.
In Algorithm~\ref{alg2}, we first compute the results of intermediate aggregation nodes and save the results in $\widehat{a}_v$ (line 5-6).
We then compute the neighborhood aggregations (i.e., $a^{(k)}_v$) for nodes in the input graph using the intermediate aggregation results $\widehat{a}_v$.

\paragraph{Memory overhead.}
Although Algorithm~\ref{alg2} includes new intermediate variables $\widehat{a}_v$, the memory overhead for storing $\widehat{a}_v$ is negligible since $\widehat{a}_v$ is not used for back propagation and can be saved in a constant memory across all GNN layers.
In the experiments, we show \xgs can increase the training throughput by $2.8\times$ at the cost of 0.1\% memory overhead.

We define a GNN-graph $\m{G}$ and a \xg $\mw{G}$ to be {\em equivalent} for a GNN model if (1) the GNN model outputs the same activations (i.e., $h^{(k)}_v$) at each GNN layer, and (2) the GNN model computes the same gradients for all trainable parameters in back propagation. 
We can use equivalent graphs interchangeably for both inference and training, since equivalent graphs produce the same outputs and gradients by definition.
Theorem~\ref{thm2} provides a necessary and sufficient condition for graph equivalence. We prove the theorem in Appendix.
\begin{theorem}
\label{thm2}
A GNN-graph $\m{G}$ and a \xg $\mw{G}$ are equivalent if and only if $\mathcal{N}(v) = \er{cover}(v)$ for all $v \in \m{V}$, where $\mathcal{N}(v)$ is $v$'s neighbors in the input graph and $\er{cover}(\cdot)$ is defined in Equation~\ref{eqn1} and~\ref{eqn2}.
\end{theorem}

%\textbf{Equality between standard graphs and hierarchical aggregation graphs.}

Equivalent graphs achieve the same model accuracy but have different runtime performance. 
Theorem~\ref{thm2} provides an efficient way to check equivalence between GNN-graphs and \xgs, and can be used as an oracle to search for optimized \xgs for any GNN-graph.

%% file: method.tex
\section{\xg Search Algorithm}
For an arbitrary GNN model and an input GNN-graph, our goal is to find an equivalent \xg with optimized runtime performance.
We define a realistic cost function to quantitatively evaluate the runtime performance of arbitrary \xgs, \hide{(\Cref{subsec:cost})} and introduce a \xg search algorithm that automatically finds an optimized \xg with the following theoretical guarantees: \hide{(\Cref{subsec:greedy})}
%Theoretically, we show that:
\begin{itemize}[leftmargin=*]
\item For GNNs with sequential \textproc{Aggregate}, the \xg search algorithm can find {\em globally optimal} \xgs under the cost function.
\item For GNNs with set \textproc{Aggregate}, finding an optimal \xg is NP-hard by a reduction from the NP-hard {\em maximum coverage problem} (see Appendix for the proof). 
The search algorithm finds at least a {\em $(1-1/e)$-approximation} of globally optimal \xgs based on the submodularity property~\cite{mossel2007submodularity}.
\end{itemize}
%Our \xg search algorithm can empirically find highly optimized \xgs for GNNs with set \textproc{Aggregate}.
%By using the cost function, the greedy algorithm can find globally optimal \xgs for GNNs with sequential \textproc{Aggregate} and highly efficient \xgs for GNNs with unordered \textproc{Aggregate}.
\subsection{Cost Function}
\label{subsec:cost}
We introduce a realistic cost function that quantitatively evaluates the runtime performance of a \xg by measuring the computation cost to perform one epoch GNN training on the \xg.

The computation cost of a GNN model includes aggregating the neighbors of each node by calling \textproc{Aggregate} and updating the activations of each node via \textproc{Update}, as shown in Algorithm~\ref{alg2}. 
For a GNN model $\m{M}$, we assume the cost of performing \textproc{Aggregate} on two elements is $\alpha_{\m{M}}$, and the cost of computing an \textproc{Update} is $\beta_{\m{M}}$.
%For a GNN model $\mathcal{M}$, we assume the cost of performing \textproc{Aggregate} on two elements is $\alpha$, and the cost of performing an \textproc{Update} is $\beta$.
In Algorithm~\ref{alg2}, computing $\widehat{a}_v$ with $|\mathcal{\widehat{N}}_v|$ neighbors requires performing $(|\mathcal{\widehat{N}}_v|-1)$ binary aggregations, whose cost is $\alpha_{\m{M}}\times(|\mathcal{\widehat{N}}_v|-1)$.
Therefore, the total computation cost of training a GNN model $\mathcal{M}$ on a \xg $\mw{G}$ is
\begin{eqnarray*}
\er{cost}(\mathcal{M}, \mathcal{\widehat{G}}) & = &\sum_{v \in \m{V} \cup \m{V}_A} \alpha_{\m{M}} (|\mathcal{\widehat{N}}_v| - 1) + \sum_{v \in \m{V}} \beta_{\m{M}}
%& = & \alpha_{\m{M}} \big(|\mw{E}| - |\m{V}| - |\m{V}_A|\big) + \beta_{\m{M}} |\m{V} |\\
 =  \alpha_{\mathcal{M}}\big(|\mw{E}| - |\m{V}_A|\big) + (\beta_{\mathcal{M}} - \alpha_{\mathcal{M}}) |\m{V}|
\end{eqnarray*}
%where $\alpha_{\mathcal{M}} = \sum_{l \in \mathcal{M}}\alpha_l$ and $\beta_{\mathcal{M}} = \sum_{l \in \mathcal{M}} \beta_l$ are the overall aggregation and update costs of $\mathcal{M}$, respectively.
%The last equation is because $\mathcal{\widehat{G}}$ is equivalent to the ordinary graph $\mathcal{G}$ only if $\mathcal{\widehat{V}}_D = \mathcal{V}$.
Since $|\mathcal{V}|$ is determined by the input graph, our goal is to minimize $\big(|\mathcal{\widehat{E}}| -  |\mathcal{\widehat{V}}_A| \big)$ as much as possible.

\subsection{Search Algorithm}
\label{subsec:greedy}
\begin{algorithm}[t]
%\footnotesize
\caption{A \xg search algorithm to automatically find an equivalent \xg for a GNN-graph with optimized runtime performance.
\textproc{Redundancy}($v_1, v_2, \mw{E}$) calculates the number of nodes aggregating both $v_1$ and $v_2$.
$\m{V}_A$ is the set of aggregation nodes in a \xg. 
Recall that $\er{cover}(u)$ is an ordered list for sequential \textproc{Aggregate} (see Equation~\ref{eqn2}).}
\label{alg3}
\begin{algorithmic}[1]
\State {\bf Input: } A GNN-graph $\mathcal{G}$ and a GNN model $\m{M}$.
%the maximum number of aggregation nodes $\er{capacity}$, the maximum depth of aggregation nodes $\er{depth}$.
\State {\bf Output: } An equivalent \xg 
\State 
\Function{Redundancy}{$v_1, v_2, \mw{E}$}
\If {$\m{M}$ has a set \textproc{Aggregate}}
\State $\m{R} = \{ u | (v_1, u) \in \mw{E} \wedge (v_2, u) \in \mw{E}\}$
\Else
\State $\m{R} = \{ u | v_1 = \er{cover}(u)[1] \wedge v_2 = \er{cover}(u)[2]\}$
\EndIf
\State \textbf{return} $|\mathcal{R}|$
\EndFunction
%\Function{Depth}{$v$}
%\State \textbf{return} $\max\{\Call{Depth}{u} + 1 | (u, v) \in \mw{E}\}$
%\EndFunction
\State
\State $\m{V}_A \leftarrow \emptyset, \mw{E} \leftarrow \mathcal{E}$
\While {$|\m{V}_A| < \er{capacity}$}
\State $(v_1, v_2) = \argmax_{v_1, v_2}$ \Call{Redundancy}{$v_1, v_2, \mw{E}$}
\If {$\Call{Redundancy}{v_1, v_2, \mw{E}} > 1$}
\State $\m{V}_A \leftarrow \m{V}_A + \{w\}$ \Comment{where $w$ is a new node}
\State $\mw{E} \leftarrow \mw{E} + (v_1, w) + (v_2, w)$
\For {$u \in \m{V}$}
\If {$(v_1, u) \in \mw{E} \wedge (v_2, u) \in \mw{E}$}
\State $\mw{E} \leftarrow \mw{E} - (v_1, u) - (v_2, u) + (w, u)$
\EndIf
\EndFor
\EndIf
\EndWhile
\State {\bf return } $(\m{V}_A \cup \m{V}, \mw{E})$
\end{algorithmic}
\end{algorithm}

%\begin{figure}
%\centering
%\subfloat[Input graph.]{
%\includegraphics[scale=0.3]{figures/graph_example.pdf}
%}
%\\
%\subfloat[Initial \xg with 9 binary aggregations.]{
%\includegraphics[scale=0.3]{figures/step0.pdf}
%}
%\\
%\vspace{-2mm}
%\subfloat[Updated \xg with 7 binary aggregations.]{
%\includegraphics[scale=0.3]{figures/step1.pdf}
%}
%\\
%\vspace{-2mm}
%\subfloat[Updated \xg with 6 binary aggregations.]{
%\includegraphics[scale=0.3]{figures/step2.pdf}
%}
%\\
%\vspace{-2mm}
%\subfloat[Final \xg with 5 binary aggregations.]{
%\includegraphics[scale=0.3]{figures/step3.pdf}
%}
%\vspace{-2mm}
%\caption{Iteratively constructing an \xg with the graph search algorithm. nodes with a red box indicate the chosen nodes at each iteration.}
%\label{fig:greedy}
%\end{figure}
%\ZJ{Say why it is hard to find an optimal solution}
We present a \xg search algorithm that finds a globally optimal \xg for GNNs with sequential \textproc{Aggregate} and a $(1-1/e)$-approximation of globally optimal \xgs for GNNs with set \textproc{Aggregate}.
In addition to an input GNN-graph and a GNN model, the algorithm also takes a hyper-parameter {\em capacity}, defining an upper limit on the number of intermediate aggregation nodes (i.e., $|\m{V}_A|$).
%that specify the maximum capacity and depth of all aggregation nodes in the $\m{V}_A$, respectively.
%The depth of a node $v$ is the length of the longest path from a node to $v$, which describes the latency to compute node $v$ since all nodes along the longest path must be computed sequentially.
%The capacity is an upper limit on $|\m{V}_A|$.

Algorithm~\ref{alg3} shows the pseudocode of the \xg search algorithm.
We start with an input GNN-graph, and iteratively insert aggregation nodes into the current \xg to merge highly redundant aggregations and remove unnecessary computation and data transfers.

In each iteration, we find a binary aggregation with the highest redundancy and insert a new aggregation node $w$ in $\m{V}_A$ to represent the binary aggregation results (line 12-15).
%a pair of nodes $(v_1, v_2)$ with the highest redundancy and introduces a new node $w$ to represent the aggregation of $v_1$ and $v_2$.
All nodes containing this binary aggregation can directly use the output of $w$ without recomputing the aggregation (line 16-18).
The \xg search algorithm iteratively reduces the computation cost of the \xg by eliminating the most redundant binary aggregation in each iteration.
%All nodes that originally includes both $v_1$ and $v_2$ as in-neighbors now contains $w$ as an input. 
%This eliminates the redundant computation for aggregating $v_1$ and $v_2$ multiple times for different nodes.
%Figure~\ref{fig:greedy} demonstrates how the graph search algorithm iteratively generates an \xg for the ordinary graph in Figure~\ref{}.

For a GNN model with a sequential \textproc{Aggregate}, Theorem~\ref{thm3} shows that ~\Cref{alg3} finds an equivalent \xg with globally optimal computation cost. We prove the theorem in Appendix.

\begin{theorem}
\label{thm3}
For any GNN-graph $\m{G}=(\m{V}, \m{E})$ and any GNN model $\m{M}$ with a sequential \textproc{Aggregate}, Algorithm~\ref{alg3} returns an equivalent \xg with globally minimized cost as long as $\er{capacity}\geq |\m{E}|$.%, where $|\m{E}|$ is the number of edges in $\m{G}$.
%For any \xg $\mw{G}$ that is equivalent to $\mathcal{G}$, $\er{cost}(\mw{G}_0) \leq \er{cost}(\mw{G}_0)$.
\end{theorem}

%We prove the correctness of Theorem~\ref{thm3} and the graph search algorithm in Appendix. 
%Theorem~\ref{thm3} shows that the graph search algorithm can find optimal \xgs for GNN models with sequential \textproc{Aggregate}.

For a GNN model with a set \textproc{Aggregate}, Theorem~\ref{thm4} shows that \Cref{alg3} finds a \xg that is at least a $(1-1/e)$-approximation of the globally optimal \xgs. We prove the theorem in Appendix.
\begin{theorem}
\label{thm4}
For any GNN-graph $\m{G}$ and any GNN model $\m{M}$ with a set \textproc{Aggregate}, Algorithm~\ref{alg3} gives a $(1-1/e)$-approximation of globally optimal \xgs under the cost function. More specifically, let $\mw{G}$ be the \xg returned by Algorithm~\ref{alg3}, and $\mw{G}_o$ is a globally optimal \xg under the $\er{capacity}$ constraint, we have
$$
\er{cost}(\m{M}, \mw{G}) \leq \frac{1}{e} \er{cost}(\m{M}, \m{G}) + \frac{e-1}{e} \er{cost}(\m{M}, \mw{G}_o)
$$
\end{theorem}
%Empirically, the \xg search algorithm finds highly optimized \xgs for real-world graph datasets, reducing the number of aggregations by up to 6.3$\times$.
%{\bf Time Complexity.} 

%\begin{table}
%\caption{Time complexity of Algorithm~\ref{alg3}. |V| and |E| denote the number of nodes and edges in the input ordinary graph, respectively.}
%\begin{tabular}{|l|l|}
%\hline
%{\bf Step} & {\bf Time Complexity} \\
%\hline
%Initialize the heap structure & $O(|\mathcal{V}|^2)$ \\
%\hline
%Query the binary aggregations & \multirow{2}{*}{$O(\er{capacity} \times \log|\mathcal{V}|)$} \\
%with highest redundancy & \\
%\hline
%Update the heap structure & $O(|\mathcal{E}| \log|\mathcal{V}|)$\\
%\hline
%\hline
%Overall & $O(|\mathcal{V}|^2 + |\mathcal{E}| \log|\mathcal{V}|)$\\
%\hline
%\end{tabular}
%\end{table}
\paragraph{Time complexity.}
\hide{Finding the binary aggregation with the highest redundancy in each iteration could be computationally expensive, since a brute-force approach requires enumerating all node pairs.}
\hide{We use a {\em heap} to maintain the redundancy score of each potential node pair and only update the heap when we add and remove edges in $\mw{E}$.}
%Table~\ref{tab:} shows the time complexity of the graph search algorithm.
\hide{Since the depth of the heap is at most $O(\log|\m{V}|)$~\footnote{This is because there can be at most $O(|\m{V}|^2)$ node pairs.}, querying the most redundant binary aggregation and modifying $\mw{E}$ each takes $O(\log|\m{V}|)$ time.}
The overall time complexity of \Cref{alg3} is $O(\er{capacity} \times |\m{V}| + |\m{E}| \times \log|\m{V}|)$ (see Appendix for the proof).

%% file: exp.tex
\section{Experiments}
Our \xg abstraction maintains predictive performance of GNNs but leads to much faster runtime performance. 
This section evaluates the runtime performance of \xgs on five real-world graph datasets.
We evaluate \xgs along three dimensions: (a) end-to-end training and inference performance; (b) number of aggregations; and (c) size of data transfers.

\subsection{Implementation}
Existing frameworks such as TensorFlow~\cite{Tensorflow} and PyTorch~\cite{Pytorch} are designed for spatial data structures (e.g., images and text), and have limited support for irregular data structures such as graphs. 
As a result, GNN models in existing frameworks translate graph structures to sparse adjacent matrices and use matrix operations to perform GNN training.

We implemented the following operations in TensorFlow r1.13 to support GNN training with \xgs.
First, {\tt graph\_to\_hag} automatically transforms an input GNN-graph to an equivalent \xg with optimized performance.
Second, {\tt hag\_aggregate} takes a HAG and nodes' activations as inputs, and computes the aggregated activations of all nodes. 
Finally, {\tt hag\_aggregate\_grad} computes the gradients of {\tt hag\_aggregate} for back propagation.

Our implementation minimizes changes to existing GNN programs: a GNN application can directly use all \xg optimizations by only modifying a few lines of code.

\subsection{Experimental Setup}
\label{sec:exp}
\begin{wraptable}{r}{0.5\linewidth}
%\begin{table}
\centering
\vspace{-6mm}
\caption{Datasets used in the experiments.}
\vspace{-2mm}
\label{tab:datasets}
\resizebox{0.75\linewidth}{!}{
\begin{tabular}{|l|l|l|l|}
\hline
{\bf Name} & {\bf \# Nodes} & {\bf \# Edges} \\
\hline
\multicolumn{3}{|c|}{Node Classification} \\
\hline
BZR~\cite{BZR} & 6,519 & 137,734\\
%SST & & & Node Classification \\
PPI~\cite{PPI} & 56,944 & 1,612,348\\
REDDIT~\cite{GraphSAGE} & 232,965 & 57,307,946\\
\hline
\multicolumn{3}{|c|}{Graph Classification}\\
\hline
IMDB~\cite{COLLAB} & 19,502 & 197,806\\
COLLAB~\cite{COLLAB} & 372,474 & 12,288,900\\
\hline
\end{tabular}
}
%\end{table}
\end{wraptable}

%We perform all experiments in our framework due to its significant better performance than other frameworks (see Figure~\ref{fig:compare_impl}). We compare the runtime performance of \xgs with ordinary graphs on the following GNN models.
%{\bf GNN models.} 
%We compare the performance of hierarchical aggregation graphs with ordinary graphs on all GNN models in Table~\ref{tab:gnns}. 

\paragraph{Datasets.} %We evaluate the performance of \xg on five real-world datasets.
Table~\ref{tab:datasets} summarizes the public datasets used in our experiments.
BZR is a chemical compound dataset, where each node is an atom and an edge is a chemical bond between two atoms~\cite{BZR}.
PPI contains a number of protein-protein interaction graphs, each of which corresponds to a different human tissue~\cite{PPI}.
REDDIT is an online discussion forum dataset, with each node being a Reddit post and each edge being commenting relations. For both PPI and REDDIT, we directly use prepossessed data from~\citet{GraphSAGE}.
IMDB and COLLAB are two collaboration datasets for graph classification~\cite{COLLAB}.
IMDB is a movie collaboration dataset, with each node representing an actor/actress, while COLLAB is a scientific collaboration dataset, with each node representing a researcher.

%{\bf Baselines.} We compare the runtime performance of \xgs with TensorFlow and DGL that uses sparse matrix operations to perform GNN operations on standard computation graphs. 

All experiments were performed running TensorFlow r1.13 on NVIDIA Tesla V100 GPUs.
Following previous work~\cite{GCN, GraphSAGE}, each GNN model has two GNN layers and one SoftMax layer.
For graph classification datasets, each GNN model also includes a mean-pooling layer to gather graph-level activations.
%We set the number of hidden dimensions to 15.
For all experiments, we set the maximum \er{capacity} of $|\m{V}_A|$ in a \xg to be $|\m{V}| / 4$, which achieves high performance on real-world graphs.
%Section~\ref{subsec:eval_para} studies how different capacities affect the runtime performance of \xgs.
\hide{In all experiments, the memory overhead to save intermediate aggregation results is negligible: intermediate nodes consume 6MB of memory in the worst case while GNN training requires more than 7GB of memory ($\sim$0.1\% memory overhead). }

\begin{wrapfigure}{r}{0.5\linewidth}
    \centering
    \vspace{-10mm}
    \includegraphics[scale=0.32]{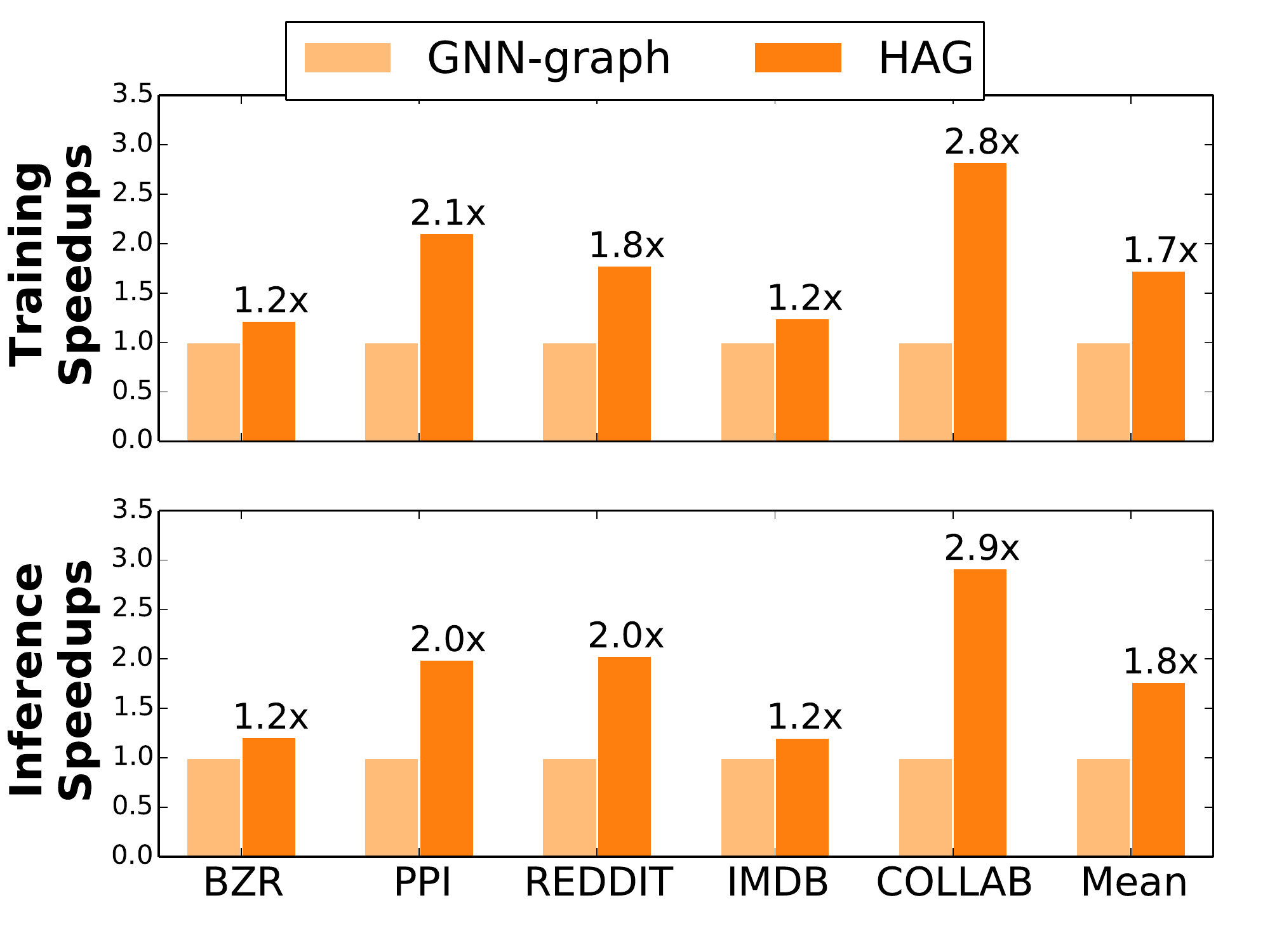}
    \vspace{-2mm}
    \caption{End-to-end performance comparison between GNN-graphs and \xgs. 
    We measure the per-epoch training time and inference latency on a 2-layer GCN model with 16 hidden dimensions in each layer.
    The performance numbers are normalized by the GNN-graph numbers.}
    \vspace{-6mm}
    \label{fig:compare_training}
\end{wrapfigure}

\subsection{End-to-End Performance}
\label{subsec:eval_end}
%\rex{Here we measure the performance by the time to compute the forward and backward propagation for XXX? (remind that this is not the accuracy kind of performance, so algorithm people do not get confused)}
We first measure the per-epoch training time and inference latency to run a 2-layer GCN model on different graph datasets. We follow previous work~\cite{GraphSAGE, BZR, COLLAB} to split the datasets into training/validation/testing sets, and use the testing sets to measure the inference latency.

%We first compare the end-to-end training performance between GNN-graphs and \xgs.
%For GNN-graphs, we also ran experiments on TensorFlow (v1.12) and DGL with the PyTorch backend (v1.0) to compare the time it takes to complete one epoch of training using different frameworks.
Figure~\ref{fig:compare_training} compares the per-epoch training time and inference latency between GNN-graphs and \xgs.
%By using the same GNN-graphs, \sys outperforms TensorFlow and DGL with the PyTorch backend by 3.7-5.5$\times$.
%The performance improvement is achieved by a number of critical system optimizations enabled in \sys, as discussed in Section~\ref{sec:impl}.
%
Compared to GNN-graphs, \xgs can improve the training and inference performance by up to 2.8$\times$ and 2.9$\times$, respectively, while maintaining the same network accuracy.
We note this improvement is achieved completely automatically, and computing a \xg is inexpensive.
Thus, because the improvement is essentially for free, we believe there is no reason not to use \xgs in preference to GNN-graphs.
%\rex{Training and inference speedup seems identical which is reasonable for random split. If you need space maybe drop one and just say they are identical?}
%The speedup is achieved by eliminating redundant computation and unnecessary data transfers in GNN computation.

%We now compare the end-to-end training performance of ordinary graphs and \xgs on two GNN models. Figure~\ref{fig:compare_training} shows the comparison results.
%Both GCN and GCN-P use unordered aggregations, and GraphSAGE-LSTM uses sequential aggregations.
%GraphSAGE-LSTM requires an ordering on each node's neighbors. For each node, we order its neighbors by their degrees. 
%Compared to directly training GNN models on ordinary graphs, \xgs maintain the same network accuracy while reducing the end-to-end training time by up to 47\%. 
%The performance improvement is achieved by eliminating redundant computation and reducing unnecessary memory accesses in neighborhood aggregations.
%This shows that \xg provides a more efficient graph representation to train GNNs.

\begin{figure}[t]
    \vspace{-4mm}
    \centering
    \subfloat[Set Aggregations.]{
    \includegraphics[scale=0.3]{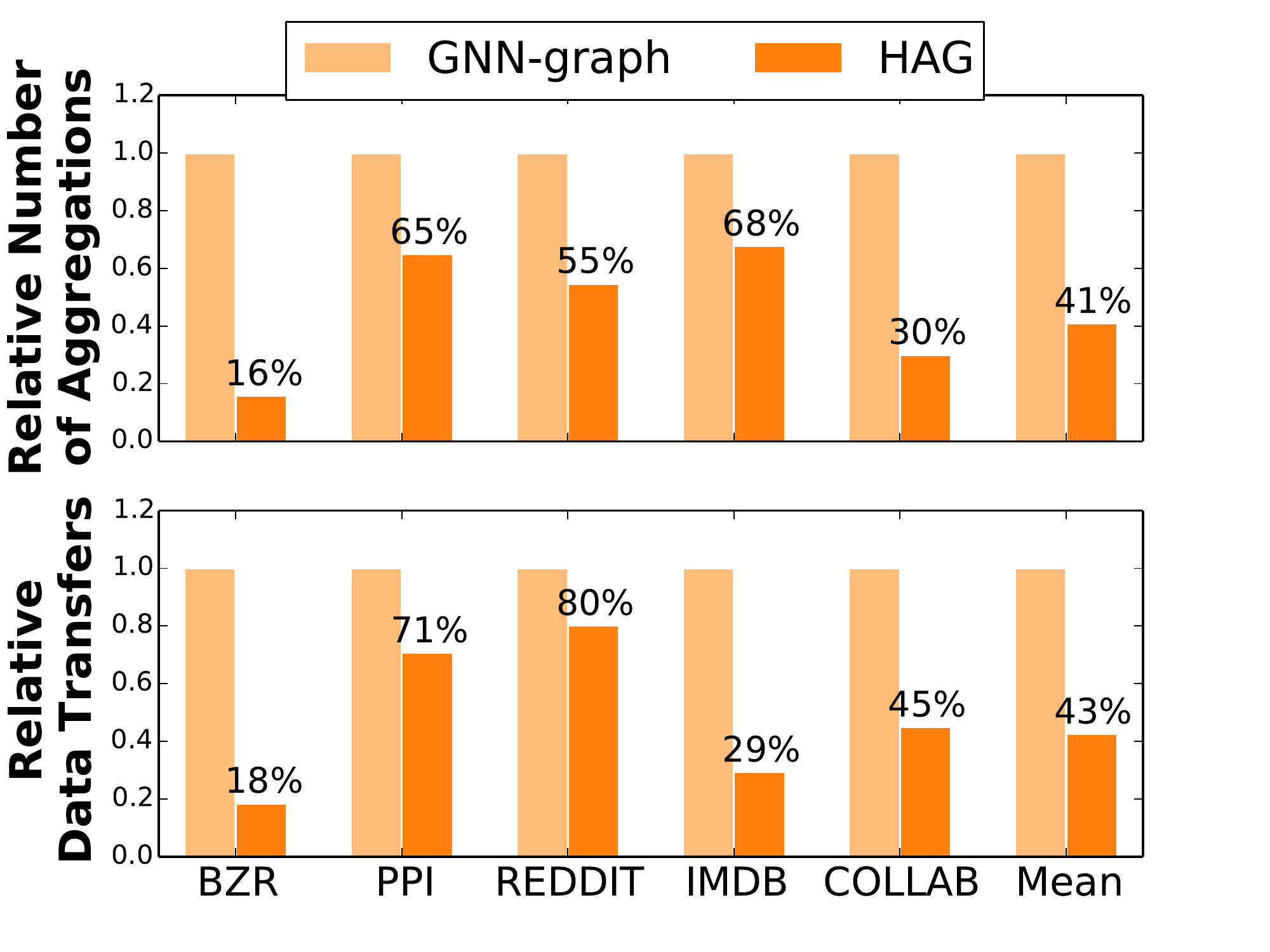}
    }
    \subfloat[Sequential Aggregations.]{
    \includegraphics[scale=0.3]{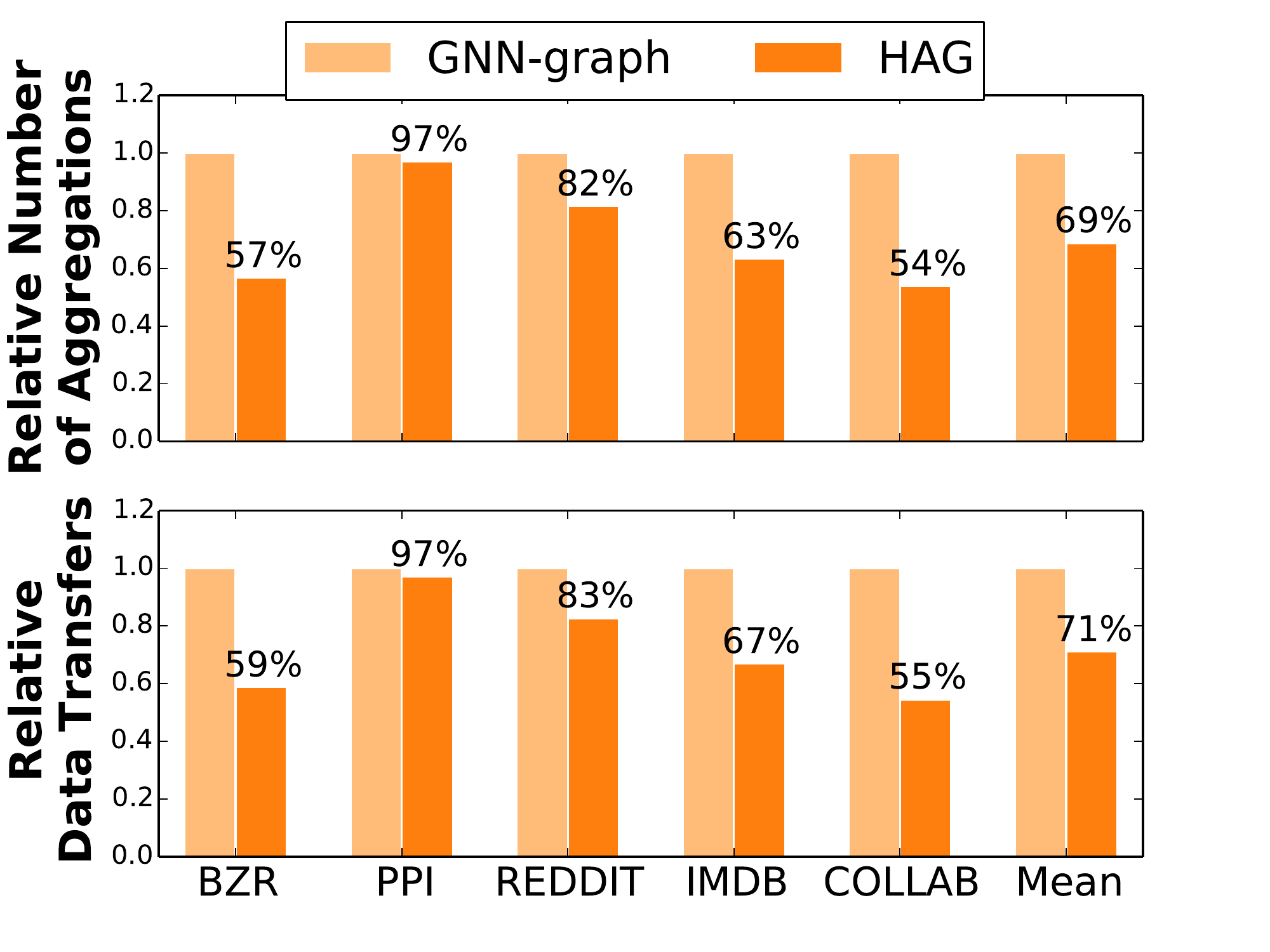}
    }
    \vspace{-2mm}
    \caption{Comparing the number of aggregations and amount of data transfers between GPU threads to perform aggregations (lower is better). 
    %Computation costs and data transfers to perform neighborhood aggregations on various computation graphs (lower is better).
    %The computation costs are measured by the numbers of binary aggregations. 
    The y-axes are normalized by GNN-graphs, and the last column in each figure is the geometry mean over all datasets.
    }
    %Runtime performance comparison between ordinary graphs and \xgs on different types of aggregations. 
    %The y-axis shows the relative numbers of binary aggregations involved in each graph representation.
    \label{fig:comapre_aggregation}
    \vspace{-6mm}
\end{figure}

\subsection{Aggregation Performance}
\label{subsec:eval_agg}
We further compare the aggregation performance of GNN-graphs and \xgs on the following two metrics: (1) the number of binary aggregations performed in each GNN layer; and (2) the size of data transfers between GPU threads to perform the aggregations.
Note that aggregating a neighbor's activations requires transferring the activations from GPU global memory to a thread's local memory.
%\rex{can we remind reader how to calculate the amount of data transfer? seems not a familiar concept for non-system reader}

%\jure{Here you need to explain the evaluation metrics. You need to define what does it mean to count the number of aggregations and (more importantly) what is the number of data transfers. Explain precisely what do we measure.}

%We further analyze the performance of \xgs and standard graphs by comparing the computation costs and data transfers to perform neighborhood aggregations in each computation graph.
%The computation cost to perform neighborhood aggregations is measured by the number of binary aggregations involved in the graph.

%Training a GNN model on ordinary graphs and their equivalent \xgs achieve the same model accuracy, and they only differ in the neighborhood aggregation scheme.
%Therefore, we first evaluate the aggregation performance on ordinary graphs and their equivalent \xgs found by the greedy algorithm.
%We compare the computation costs and memory accesses to perform neighborhood aggregations.

Figure~\ref{fig:comapre_aggregation} shows the comparison results.
For GNNs with set aggregations, \xgs reduce the number of aggregations by 1.5-6.3$\times$ and the size of data transfers by 1.3-5.6$\times$. 
For GNNs with sequential aggregations, \xgs reduce aggregations and data transfers by up to 1.8$\times$ and 1.9$\times$, respectively.

Although the search algorithm finds a globally optimal \xg for sequential aggregations (Theorem~\ref{thm3}) and a $(1-1/e)$-approximation of globally optimal \xgs for set aggregations (Theorem~\ref{thm4}), we observe the performance improvement is more significant for set aggregations.
%Optimality for \xgs with sequential and set aggregations are of course different problems.
%because set aggregations can be reordered to eliminate more redundant aggregations,
Optimality for \xgs with set aggregation involves more potential redundancy compared to sequential aggregations, due to permutation invariance of set aggregation.
 Thus higher performance can be achieved with \xgs for set aggregations, though optimal solutions are more difficult to compute.
%This is because the greedy algorithm can opportunistically reorder aggregations to further eliminate redundant aggregations.

It is also worth noting that the \xg search algorithm can find highly optimized \xgs even on very sparse graphs.
For example, on the COLLAB dataset with a graph density of 0.01\%, our algorithm reduces the number of aggregations and data transfers by 3.3$\times$ and 2.2$\times$, respectively. 
%
%\vspace{-1mm}

\subsection{Capacity}
\label{subsec:eval_para}
\begin{wrapfigure}{r}{0.4\linewidth}
    \vspace{-10mm}
    \centering
    \includegraphics[scale=0.3]{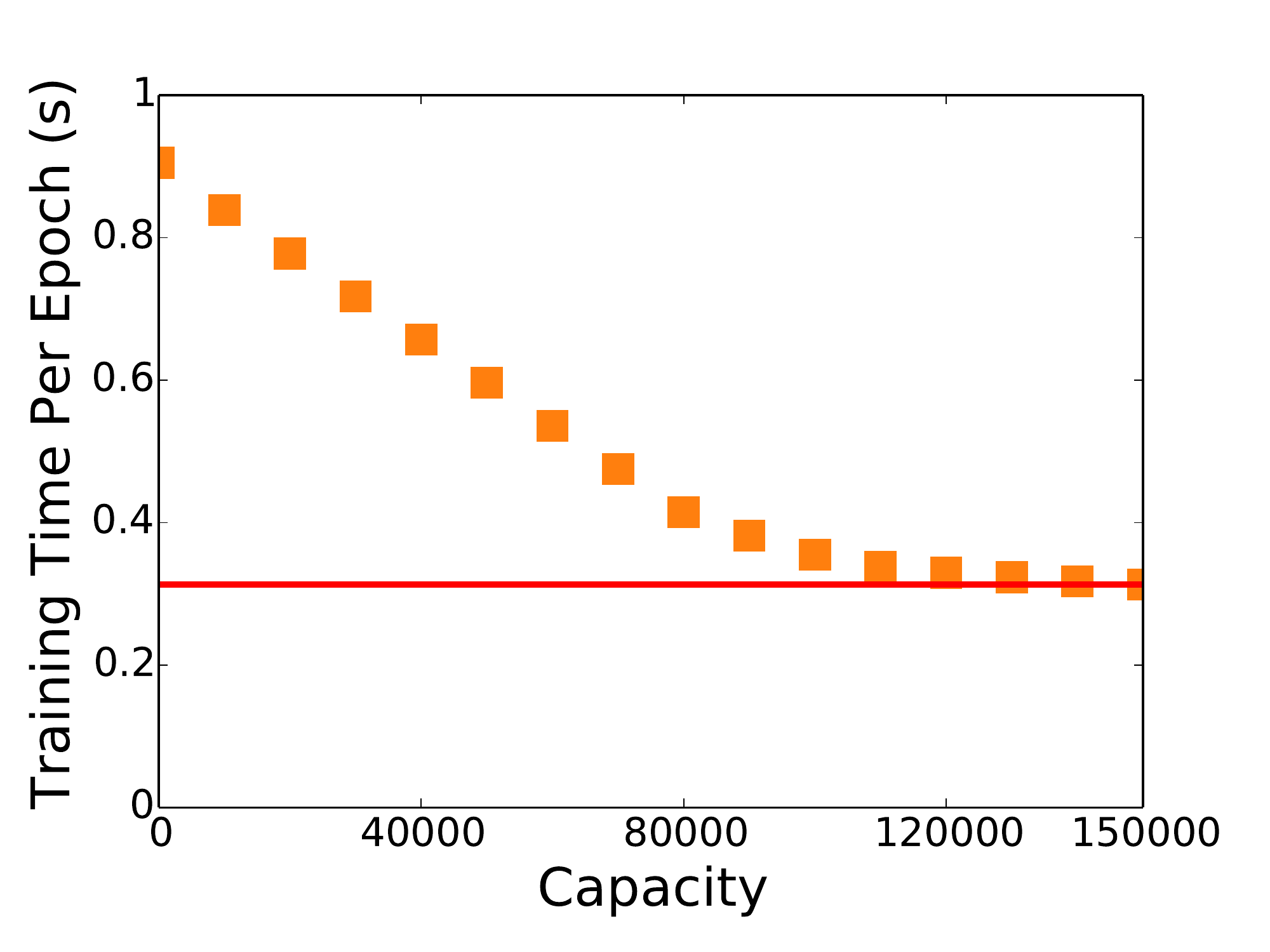}
    \vspace{-2mm}
    \caption{Comparing different \xgs and their per-epoch GCN training time on the COLLAB dataset. 
    %The squares show the training times of \xgs with different capacities. 
    The red line indicates the training time of the best discovered \xg by the search algorithm.}
    \label{fig:capacity}
    \vspace{-6mm}
\end{wrapfigure}

We study how different values of capacity affect the runtime performance of the generated \xgs. 
Recall that capacity is an upper bound on the number of aggregation nodes in a \xg.
In our HAG search algorithm, a larger value of capacity allows the algorithm to eliminate more redundant aggregations and therefore achieves lower cost.

Figure~\ref{fig:capacity} shows that a larger value of {\em capacity} can consistently improve the end-to-end training performance, which indicates that the cost function is an appropriate metric to evaluate and compare the performance of different \xgs.

By gradually increasing the capacity, the search algorithm eventually finds a \xg with $\sim$150K aggregation nodes, which consume 6MB of memory (0.1\% memory overhead) while improving the training performance by 2.8$\times$.
\vspace{-1mm}

%% file: supply.tex
\appendix
\section{Proof of Theorem~\ref{thm2}}
%\begin{theorem}
%\label{thm2}
%A GNN-graph $\m{G}$ and a \xg $\mw{G}$ are equivalent if and only if $\mathcal{N}(v) = \er{cover}(v)$ for all $v \in \m{V}$, where $\mathcal{N}(v)$ is $v$'s neighbors in the input graph.
%\end{theorem}
\begin{proof}
It is sufficient to prove that if $\mathcal{N}(v) = \er{cover}(v)$ for all $v \in \m{V}$, then the GNN-graph $\m{G}$ and the \xg $\mw{G}$ generate the same outputs (i.e., $h_v^{(k)}$) for every GNN layer. 

We prove this by induction. Assume a GNN-graph $\m{G}$ and a \xg $\mw{G}$ generate the same outputs for the ($k$-1)-th layer, we prove the two graphs produce the same outputs for the $k$-th GNN layer.

In Algorithm 2, $\widehat{a}_v$ is the aggregation results of node $v$, which is defined as
\begin{eqnarray*}
\widehat{a}_v & = & \textproc{Aggregate}(h_u^{(k-1)} | u \in cover(v)) \\
& = & \textproc{Aggregate}(h_u^{(k-1)} | u \in \m{N}(v))
\end{eqnarray*}
This proves that Algorithm 1 and Algorithm 2 compute the same $a^{(k)}_v$. 
In addition, both algorithms use the same $\textproc{Update}$ function that takes $a^{(k)}_v$ and $h^{(k-1)}_v$ as inputs and computes $h^{(k)}_v$, which applies that the two algorithms compute the same $h^{(k)}_v$.
\end{proof}

\section{Proof of Theorem~\ref{thm3}}
%\begin{theorem}
%\label{thm3}
%For any GNN-graph $\m{G}$ and any GNN model $\m{M}$ with a sequential \textproc{Aggregate}, Algorithm 3 returns an equivalent \xg with globally minimized cost as long as 
%$\er{capacity}\geq |\m{E}|$, where $|\m{E}|$ is the number of edges in $\m{G}$.
%For any \xg $\mw{G}$ that is equivalent to $\mathcal{G}$, $\er{cost}(\mw{G}_0) \leq \er{cost}(\mw{G}_0)$.
%\end{theorem}
\begin{proof}
Sequential aggregations require a specific ordering of a node's neighbors. Let $\m{N}_v$ denote the ordered list of node $v$'s neighbors and $\m{L}_v^{(i)}$ denote a list of the first $i$ elements in $\m{N}_v$:
$$
\m{L}_v^{(i)} = \big(\m{N}_v(1), \m{N}_v(2), ..., \m{N}_v(i)\big)
$$
where $\m{N}_v(i)$ is the $i$-th neighbor of node $v$.

$\m{L}_v^{(i)}$ represents a necessary intermediate aggregation step for computing $a^{(k)}_v$ (since sequential aggregations are not commutative), and therefore any \xg must compute $\m{L}_v^{(i)}$ as an intermediate aggregation.
Counting the number of distinct $\m{L}_v^{(i)}$ (where $v\in\m{V}$ and $1 < i \leq |\m{N}_v|$) provides a lower bound on the number of aggregations any equivalent \xg must perform. Assuming $\mw{G}_o$ is a globally optimal \xg under the cost model, we have:
$$
\er{cost}(\m{M}, \mw{G}_o) \geq \alpha_{\m{M}} \times \er{lb} + (\beta_{\m{M}} - \alpha_{\m{M}}) |\m{V}|
$$
where $\er{lb}$ is the number of distinct $\m{L}_v^{(i)}$ that must be computed by any equivalent \xg.

Assuming $\mw{G}$ is the output \xg of Algorithm 3, we prove that  $\er{cost}(\m{M}, \mw{G}) = \er{cost}(\m{M}, \mw{G}_o)$ by using contradiction.
In the case $\er{cost}(\m{M}, \mw{G}) > \er{cost}(\m{M}, \mw{G}_o)$, $\mw{G}$ must perform more than $lb$ aggregations. 

{\bf Case 1.} 
One possible case is that $\mw{G}$ computes at least one aggregation that is not a prefix of any $\m{N}_v$, indicating that $\mw{G}$ performs some useless aggregations, which contradicts with the fact that all intermediate aggregations added to $\mw{G}$ must be used at least once.

{\bf Case 2.}
The other possible case is that $\mw{G}$ computes the aggregation of some $\m{L}_v^{(i)}$ multiple times.
However, in Algorithm 3, each iteration reduces the number of aggregations by at least 1, and there are $|\m{E}|$ aggregations initially. 
This implies there cannot be redundant aggregations after $|\m{E}|$ iterations, which contradicts with the precondition of Case 2.
\end{proof}

\section{Proof of Theorem~\ref{thm4}}
%\begin{theorem}
%\label{thm4}
%For any GNN-graph $\m{G}$ and any GNN model $\m{M}$ with a set \textproc{Aggregate}, Algorithm 3 gives a $(1-1/e)$-approximation under the cost model. More specifically, assume $\mw{G}$ is the \xg returned by Algorithm 3 and $\mw{G}_o$ is a globally optimal \xg, then
%$$
%\er{cost}(\m{M}, \mw{G}) \leq \frac{1}{e} \er{cost}(\m{M}, \m{G}) + \frac{e-1}{e} \er{cost}(\m{M}, \mw{G}_o)
%$$
%\end{theorem}

\begin{proof}
The idea of the proof is to build a {\em monotone submodular function}~\cite{IntroAlg} based on the cost model. 

For any GNN-graph $\m{G}$ and an equivalent $\mw{G}$, we define
\begin{eqnarray}
\label{eqn0}
f(\mw{G}) & = & \er{cost}(\m{M}, \m{G}) - \er{cost}(\m{M}, \mw{G})  \\
& = & \alpha_{\m{M}} (|\m{E}| - |\mw{E}| + |\m{V}_A|) 
\end{eqnarray}
where $\m{V}_A$ is the set of aggregation nodes in $\mw{G}$, and $\m{E}$ and $\mw{E}$ are the set of edges in $\m{G}$ and $\mw{G}$, respectively.
$f(\mw{G})$ measures the number of aggregations that can be saved by using $\mw{G}$ for GNN training.

We begin by defining the subset relations between different \xgs. For two \xgs $\mw{G}$ and $\mw{G}'$, we define $\mw{G} \subseteq \mw{G'}$ iff $\m{V}_A$ is a subset of $\m{V}_A'$, where $\m{V}_A$ and $\m{V}_A'$ are the aggregation nodes in $\mw{G}$ and $\mw{G}'$, respectively.

{\bf Prove that $f(\mw{G})$ is monotone.} We show that for all $\mw{G} \subseteq \mw{G'}$, $f(\mw{G}) \leq f(\mw{G'})$. This is true since $\mw{G} \subseteq \mw{G'}$ indicates that $\mw{G}'$ contains all aggregation nodes in $\mw{G}$, which applies that $\mw{G}'$ can at least save the same number of aggregations as $\mw{G}$.

{\bf Prove that $f(\mw{G})$ is submodular.} We show that for all $\mw{G} \subseteq \mw{G'}$ and any aggregation node $n$, $f(\mw{G} + \{n\}) - f(\mw{G}) \geq f(\mw{G}' + \{n\}) - f(\mw{G}')$.
This inequality holds because $f(\mw{G} + \{n\}) - f(\mw{G})$ measures the number of aggregations we can further save by adding aggregation $n$ to the existing \xg, which monotonically decreases as we add more aggregation nodes to the \xg.

Let $\mw{G}^{(i)}$ denote the result \xg after the $i$-th iteration of Algorithm 3. $\mw{G}^{(i)}$ includes exactly $i$ aggregation nodes. Let $\mw{G}_o$ denote the optimal \xg under the cost model with $k$ aggregation nodes. We claim via induction that for $0\leq i \leq k$,
\begin{equation}
\label{eqn1}
f(\mw{G}_o) - f(\mw{G}^{(i)}) \leq (1 - 1/k)^i f(\mw{G}_o) 
\end{equation}

The base case is trivially true. In the $i$-th step, Algorithm 3 selects an aggregation node $a_i$ by maximizing the marginal gain $f(\mw{G}^{(i)} + a_i) - f(\mw{G}^{(i)})$. Observe that the remaining aggregation nodes includes $\mw{G}_o \setminus \mw{G}^{(i-1)}$, a set of at most $k$ elements. The submodularity applies that
$$
f(\mw{G}_o) - f(\mw{G}^{(i-1)}) \leq \sum_{a \in \mw{G}_o \setminus \mw{G}^{(i-1)}} \big( f(\mw{G}^{(i)} + a) - f(\mw{G}^{(i)} \big)
$$
and this implies that the aggregation node $a_i$ has marginal value
\begin{eqnarray*}
& & f(\mw{G}^{(i-1)} + a_i) - f(\mw{G}^{(i-1)}) \\
&\geq &\frac{1}{|\mw{G}_o \setminus \mw{G}^{(i-1)}|}\sum_{a\in \mw{G}_o \setminus \mw{G}^{(i-1)}}{\big( f(\mw{G}^{(i)} + a) - f(\mw{G}^{(i)} \big)} \\
&\geq & \frac{1}{k} \big( f(\mw{G}_o) - f(\mw{G}^{(i-1)})\big)
\end{eqnarray*}

Assuming that Inequality~\ref{eqn1} holds for $\mw{G}^{(i-1)}$, we have
\begin{eqnarray*}
f(\mw{G}_o) - f(\mw{G}^{(i)}) & = &f(\mw{G}_o) - f(\mw{G}^{(i-1)}) - \big( f(\mw{G}^{(i)}) - f(\mw{G}^{(i-1)}) \big)\\
& \leq & f(\mw{G}_o) - f(\mw{G}^{(i-1)} - \frac{1}{k} (f(\mw{G}_o) - f(\mw{G}^{(i-1)})) \\
& = & (1 - 1/k) (f(\mw{G}_o) - f(\mw{G}^{(i-1)})) \\
& \leq & (1-1/k)^i f(\mw{G}_o)
\end{eqnarray*}
which proves Inequality~\ref{eqn1}. Therefore, we have
$$
f(\mw{G}_o) - f(\mw{G}^{(k)}) \leq (1-1/k)^k f(\mw{G}_o) \leq e^{-1} f(\mw{G}_o)
$$
By taking in the definition of $f(\cdot)$, we have
$$
\er{cost}(\m{M}, \mw{G}) \leq \frac{1}{e} \er{cost}(\m{M}, \m{G}) + \frac{e-1}{e} \er{cost}(\m{M}, \mw{G}_o)
$$
\end{proof}

\section{Time Complexity of Algorithm 3}
\begin{theorem}
The overall time complexity of Algorithm 3 is $O(\er{capacity} \times |\m{V}| + |\m{E}| \times \log|\m{V}|)$.
\end{theorem}

\begin{proof}
We use a {\em heap} to maintain the redundancy score of each potential node pair and only update the heap when we add and remove edges in $\mw{E}$.
%Table~\ref{tab:} shows the time complexity of the graph search algorithm.
Since the depth of the heap is at most $O(\log|\m{V}|)$~\footnote{This is because there can be at most $O(|\m{V}|^2)$ node pairs.}, querying the most redundant binary aggregation and modifying $\mw{E}$ each takes $O(\log|\m{V}|)$ time.

First, we calculate the number of queries and updates to the heap structure:
\begin{itemize}
\item The algorithm iteratively pull the most redundant binary aggregation from the heap and add it to $\m{V}_A$. Since the number of vertices in $\m{V}_A$ is smaller than $\er{capacity}$, the total number of queries is $O(\er{capacity})$.
\item The algorithm inserts two new edges into $\mw{E}$ in line 16 and removes one edge from $\mw{E}$ in line 19. Since line 16 can be invoked at most $O(\er{capacity})$ times, the total number of invocations to line 19 is $O(|\m{E}| + 2\times \er{capacity})$. Therefore, the overall number of updates is $O(|\m{E}| + \er{capacity})$.
\end{itemize}

Second, the enumeration over all vertices in $\m{V}$ (line 17) involves time complexity of $O(\er{capacity} \times |\m{V}|)$. Therefore, the overall time complexity of Algorithm 3 is 
\begin{eqnarray*}
& & O\big(\er{capacity} \times |\m{V}| + (|\m{E}| + \er{capacity} ) \times \log|\m{V}|\big) \\
& & = O(\er{capacity} \times |\m{V}| + |\m{E}| \times \log|\m{V}|)
\end{eqnarray*}
\end{proof}